%% file: main.tex
\documentclass[letterpaper]{article} 
\usepackage{aaai2027}  
\usepackage[hyphens]{url}  
\usepackage{graphicx} 
\usepackage{natbib}  
\usepackage{caption} 
\usepackage{algorithm}
\usepackage{algorithmic}

\usepackage{newfloat}
\usepackage{listings}
\DeclareCaptionStyle{ruled}{labelfont=normalfont,labelsep=colon,strut=off} 
\floatstyle{ruled}
\newfloat{listing}{tb}{lst}{}
\floatname{listing}{Listing}

\usepackage{booktabs}
\usepackage{multirow}
\usepackage{amsthm}
\usepackage{amsmath}
\usepackage{amssymb}
\usepackage{bm}
\usepackage{xcolor}
\usepackage{tcolorbox}
\tcbuselibrary{skins}

\definecolor{deltapos}{HTML}{2E7D32}  
\definecolor{deltaneg}{HTML}{C62828}  
\makeatletter
\def\deltapct@dispatch#1#2\deltapct@stop{%
  \ifx#1-%
    \textcolor{deltaneg}{$\blacktriangledown$\,$#2$}%
  \else\ifx#1+%
    \textcolor{deltapos}{$\blacktriangle$\,$#2$}%
  \else
    \textcolor{deltapos}{$\blacktriangle$\,$#1#2$}%
  \fi\fi}
\newcommand{\deltapct}[1]{%
  {\scriptsize\,\deltapct@dispatch#1\deltapct@stop}}
\newcommand{\nodelta}{\phantom{\deltapct{+0.00}}}
\makeatother
\definecolor{topone}{HTML}{A3DC9A}
\definecolor{toptwo}{HTML}{DEE791}
\definecolor{topthree}{HTML}{FFF9BD}
\newcommand{\tone}[1]{\colorbox{topone}{\textbf{#1}}}
\newcommand{\ttwo}[1]{\colorbox{toptwo}{#1}}

\newtheorem{proposition}{Proposition}

\newcommand{\calL}{\mathcal{L}}

\newcommand{\calR}{\mathcal{R}}

\newcommand{\calG}{G}
\newcommand{\calT}{\mathcal{T}}
\newcommand{\calF}{\mathcal{F}}
\newcommand{\calJ}{\mathcal{J}}
\newcommand{\calD}{\mathcal{D}}
\newcommand{\E}{\mathbb{E}}

\input{tables}

\title{Hindsight-Anchored Policy Optimization: Learning Through Hindsight with Thompson Sampling-Inspired Adaptive Gating}

\author {
    Yuning Wu\thanks{Equal contribution},
    Ke Wang\footnotemark[1],
    Haoran Liu,
    Chaoqun Jia,
    Devin Chen,
    Kai Wei
}
\affiliations {
    Amazon\\
    \{yuningwu*, kewangv*, liuhr, enzojia, devichen, kaiwe\}@amazon.com
}

\begin{document}

\maketitle

\begin{abstract}
    Reinforcement Learning with Verifiable Rewards improves reasoning in large language models, yet on-policy learning often suffers from cold-start challenges in sparse-reward settings. Recent mixed-policy approaches address this by combining off-policy teacher data with on-policy training. However, simply combining these introduce a persistent off-policy gradient mass that risks training collapse and instability. To address this challenge, we propose Hindsight-Anchored Policy Optimization (HAPO), a framework that allows teacher intervention to act as a temporary support. HAPO employs \emph{Beta-Binomial confidence gating}, an adaptive gating mechanism that decides when to open the gate for teacher intervention. The intervention operates with \emph{Synthetic Success Injection}, which replaces the group's lowest-reward rollout with a verified teacher trajectory. We also introduce \emph{adaptive threshold annealing}, which gradually retracts the support and restores on-policy training within a finite horizon to mitigate persistent off-policy drift. We demonstrate that HAPO can be layered on top of existing mixed-policy methods in a generalizable manner. Across six math reasoning benchmarks and two model scales, HAPO improves the average accuracy of three major mixed-policy methods while maintaining training stability.
\end{abstract}


\section{Introduction}

Reinforcement Learning with Verifiable Rewards (RLVR) improves reasoning in large language models (LLMs) by optimizing for outcome correctness~\citep{lambert2024tulu}. On-policy reinforcement learning enables models to explore diverse solution paths and discover reasoning strategies beyond those present in supervised datasets~\citep{shenfeld2025rlsrazoronlinereinforcement,shao2024deepseekmath}. However, learning can be hindered in sparse-reward environments due to the lack of informative feedback during the early stages of training. A notable example is GRPO's group-relative advantage estimation: when all sampled rollouts for a prompt fail, they receive identical rewards, resulting in zero advantage for every trajectory. As a result, the policy gradient vanishes, preventing the model from learning from these failed attempts~\citep{shao2024deepseekmath}.

Recent mixed-policy approaches address the cold-start challenge by combining off-policy teacher data with on-policy training. Prior work has shown that verified demonstrations provide informative supervision for reasoning tasks~\citep{ouyang2022traininglanguagemodelsfollow,wei2022finetunedlanguagemodelszeroshot}. Recent works incorporate teacher demonstrations directly into the online training loop instead of relying on a separate supervised pretraining stage~\citep{zhang2025onpolicyrlmeetsoffpolicy,lv2026unifiedviewlargelanguage,liu2025uftunifyingsupervisedreinforcement,ma2025learningreinforcementlearningcant,huang2025blending}. These methods differ in how teacher trajectories are sampled, weighted, or scheduled, and some gradually reduce the amount of guidance~\citep{yan2025learningreasonoffpolicyguidance,fu2025srftsinglestagemethodsupervised,su2025trustregionadaptivepolicyoptimization}. To our best knowledge, all existing approaches rely on off-policy supervision throughout training.  While mixed-policy alleviates sparse-reward cold starts, it also introduces a persistent off-policy gradient mass that risks training collapse and instability. This raises a fundamental question: 
\textbf{
\emph{can teacher guidance enter as temporary, self-retracting support in sparse-reward settings, rather than remain a sustained component over the global training horizon?}
}

\begin{figure*}[t]
    \centering
    \includegraphics[width=0.85\textwidth]{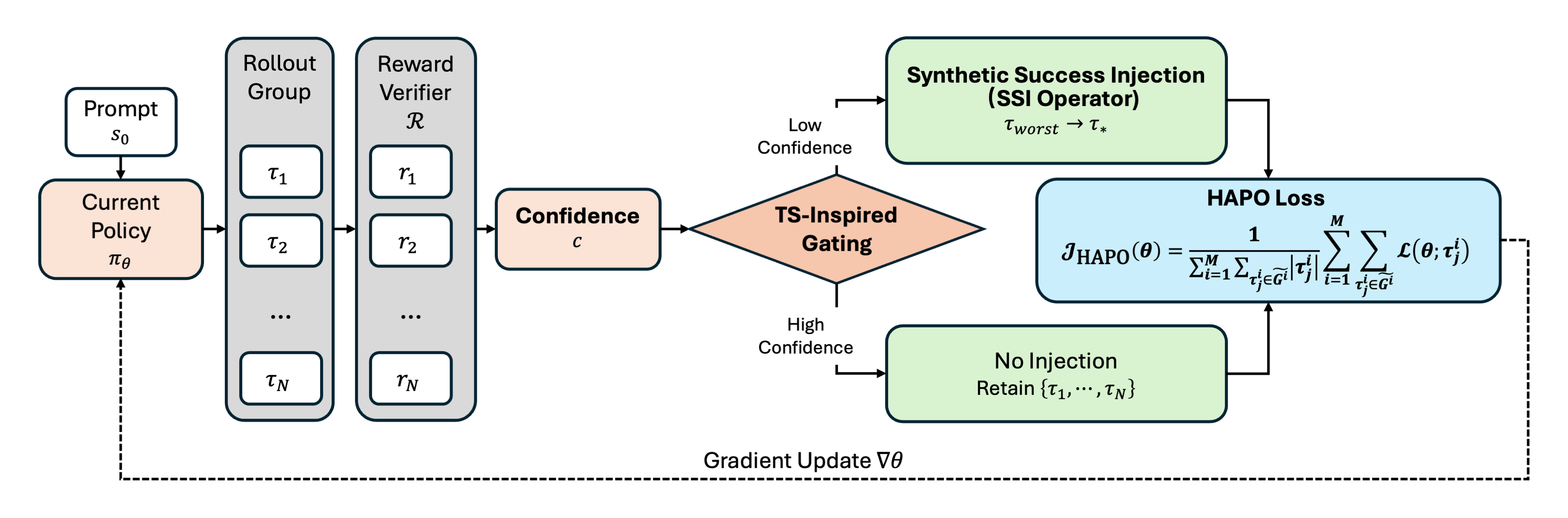}
    \caption{Hindsight-Anchored Policy Optimization (HAPO) system architecture.}
    \label{fig:hapo_arc}
\end{figure*}

To address this question, we introduce \textbf{Hindsight-Anchored Policy Optimization (HAPO)}, a framework for mixed-policy training that ties teacher intervention to the policy's current competence. HAPO targets existing methods that inject teacher trajectories into group-based policy optimization during the rollout phase. Rather than reshaping how the injected trajectory enters the objective, HAPO controls \emph{when} teacher guidance is introduced and \emph{how long} it remains active. HAPO has three components. \emph{Synthetic Success Injection (SSI)} replaces the lowest-reward trajectory of a rollout group with a verified teacher trajectory, placing a successful anchor in a group whose own rollouts yield no advantage. \emph{Beta-Binomial confidence gating (or TS-Inspired gating)}, inspired by Thompson sampling, estimates the per-prompt success rate as the posterior mean of a Beta-Binomial model and opens the gate for SSI when this estimate falls below a threshold. \emph{Adaptive threshold annealing} decides \emph{how long} guidance remains available. The schedule lowers the threshold below the smallest attainable value of the confidence estimate, after which the gate cannot reopen. This bounds the total off-policy exposure of a run and returns the update to pure on-policy training at a known epoch. We instantiate HAPO on top of LUFFY~\citep{yan2025learningreasonoffpolicyguidance}, SRFT~\citep{fu2025srftsinglestagemethodsupervised}, and TRAPO~\citep{su2025trustregionadaptivepolicyoptimization}, three representative mixed-policy methods, to demonstrate its generalizability.

\paragraph{Contributions.} \textbf{(1) A Beta-Binomial confidence gate for adaptive success injection under sparse rewards.} We characterize the condition that motivates mixed-policy injection, namely rollout groups with too few successful trajectories, whose group-relative advantages provide little learning signal during the sparse cold-start stage. From this observation we derive a gate that estimates the per-prompt success rate as a Beta--Binomial posterior mean, following the Bayesian view of Thompson sampling, and compares this estimate against a threshold to decide when the gate opens. \textbf{(2) Annealing as a measure for global-horizon training stability.} We show that controlling injection over the global horizon is necessary once supervised and RL gradients are mixed within each rollout group, and that annealing the gate to extinction at a finite, closed-form epoch keeps training stable. \textbf{(3) Consistent gains across methods and scales.} Across six math reasoning benchmarks, two base-model scales, and three recent mixed-policy host methods, HAPO improves the average accuracy of the method it augments while keeping training stable, with the largest gain of $+4.52$ points for SRFT at the Qwen2.5-7B scale.

\section{Related Work}

\paragraph{Reinforcement Learning for Reasoning.}
Post-training for LLM reasoning has increasingly shifted toward RLVR, where correctness can be checked directly using rule-based or programmatic verifiers~\citep{lambert2024tulu, tie2025surveyposttraininglargelanguage}. Optimization methods such as PPO~\citep{schulman2017proximalpolicyoptimizationalgorithms} and GRPO~\citep{shao2024deepseekmath} have shown that relatively simple reward signals can lead to substantial reasoning gains. However, recent analyses suggest that purely on-policy reasoning RL remains strongly dependent on the initialization of the base model and struggles in sparse-reward regimes where successful trajectories are rarely discovered~\citep{yue2025doesreinforcementlearningreally, yu2025dapoopensourcellmreinforcement, zeng2025simplerlzooinvestigatingtamingzero}. Our work is motivated by this cold-start limitation.

\paragraph{Combining Exploration and Imitation.} A long-standing challenge in policy optimization is how to combine exploratory RL with off-policy teacher guidance. The standard sequential recipe of SFT followed by RL is effective in practice, but suffers from distribution drift and catastrophic forgetting once the policy moves away from the teacher distribution \citep{zhang2025onpolicyrlmeetsoffpolicy, yoshihara2025practicaltwostagerecipemathematical}. More recent mixed-policy methods instead integrate teacher demonstrations directly into the RL loop through token-level, sample-level, or group-level interventions \citep{yan2025learningreasonoffpolicyguidance,fu2025srftsinglestagemethodsupervised,huang2025blending,liu2025uftunifyingsupervisedreinforcement}. These approaches improve optimization stability, but most use replacement or weighting schemes that remain active throughout training. HAPO is complementary to this line of work: rather than replacing the base method, it modifies \emph{when} and \emph{how long} teacher intervention is applied.

\paragraph{Adaptive Hybrid Post-Training.} Several recent methods introduce adaptive or staged mechanisms for balancing supervised and reinforcement signals \citep{lv2026unifiedviewlargelanguage,ma2025learningreinforcementlearningcant,fu2025srftsinglestagemethodsupervised,su2025trustregionadaptivepolicyoptimization}. These works recognize that the role of teacher data should vary over the course of training. Our framework shares this motivation, but differs in two key ways. First, HAPO operates as a method-agnostic add-on that leaves the host objective intact. Second, its schedule is specified by the group size plus a single epoch-valued knob, and retires intervention at a finite, pre-specified epoch rather than through an asymptotically decaying weight.

\paragraph{Hindsight-Based Correction.} HAPO is conceptually related to hindsight-based methods in reinforcement learning~\citep{andrychowicz2018hindsightexperiencereplay}, which reinterpret failed experiences to extract useful learning signals. These methods relabel unsuccessful trajectories with achieved outcomes, apply policy gradient updates with hindsight goals~\citep{rauber2019hindsightpolicygradients}, prioritize hindsight samples based on energy metrics~\citep{zhao2020energybasedhindsightexperienceprioritization}, or combine hindsight replay with trust region optimization~\citep{zhang2021hindsighttrustregionpolicy}. HAPO's SSI adapts this intuition to mixed-policy reasoning RL: when an entire rollout group fails, HAPO inserts a verified teacher trajectory as a hindsight anchor, converting an otherwise uninformative group into a useful optimization signal. It adapts a hindsight correction principle into the mixed-policy LLM post-training setting.

\section{Preliminaries and Problem Formulation}

In this section, we review the components needed to formalize HAPO: the MDP formulation of autoregressive language generation, Group Relative Policy Optimization (GRPO)~\citep{shao2024deepseekmath}, and the Bayesian confidence estimate that underlies our Thompson sampling-inspired adaptive gating mechanism.

\subsection{Markov Decision Process}

We formulate language generation as a finite-horizon Markov Decision Process $M=(S,A,P,R,\gamma)$ \citep{Sutton1998,murphy2025reinforcementlearningoverview,zhang2025surveyreinforcementlearninglarge}. A state $s_t \in S$ consists of the prompt together with previously generated tokens, an action $a_t \in A$ is the next token, the transition is deterministic given token generation, and we use $\gamma=1$. A trajectory is denoted by $\tau = \{s_0,a_1,\dots,s_T,a_T\}$. The objective is to learn a policy $\pi_\theta$ that maximizes expected trajectory reward:
\begin{equation}
\arg\max_\theta \calJ(\theta) = \mathbb{E}_{\tau \sim \pi_\theta(\cdot \mid s_0)}[R(\tau)]
\label{eq:objective}
\end{equation}

\subsection{Group Relative Policy Optimization}
A common approach to maximize the objective in Eq.~(\ref{eq:objective}) is Proximal Policy Optimization (PPO)~\citep{schulman2017proximalpolicyoptimizationalgorithms}. PPO maintains both an actor and a critic network, which adds computational and memory cost when training LLMs. GRPO~\citep{shao2024deepseekmath} removes the critic network and instead estimates advantages from the relative performance of grouped trajectories.

We consider a curated dataset $\mathcal{D}=\{(s_0^i, \tau_*^i) : i \in \{1, \ldots, M\}\}$, where $s_0^i$ is the prompt (initial state) and $\tau_*^i$ is the teacher trajectory. For each prompt $s_0^i$, $N$ trajectories are sampled using the old policy $\pi_{\theta_{\text{old}}}$, forming a group of samples denoted as $\calG^i = \{\tau_j^i : j \in \{1, \ldots, N\} \}$. GRPO computes the advantage of each trajectory by normalizing rewards within the group $\calG^i$:
\begin{equation}
A_j^i = \frac{\calR(\tau_j^i) - \text{mean}(\{\calR(\tau_k^i): \tau_k^i \in \calG^i\} \})}{\text{std}(\{\calR(\tau_k^i) : \tau_k^i \in \calG^i\}\})}
\label{eq:advantage}
\end{equation}

Considering the clipped surrogate objective from PPO, the GRPO aggregates over all groups:

\begin{equation}
\calJ_{\text{GRPO}}(\theta) = \frac{1}{\sum_{i=1}^M \sum_{j=1}^N |\tau_j^i|} \sum_{i=1}^M \sum_{j=1}^N \sum_{t=1}^{|\tau_j^i|} \text{CLIP}(r_{j,t}^i(\theta), A_j^i, \epsilon)
\label{eq:grpo}
\end{equation}

where $r_{j,t}^i(\theta) = \frac{\pi_\theta(\tau_{j,t}^i|s_0^i, \tau_{j,<t}^i)}{\pi_{\theta_{\text{old}}}(\tau_{j,t}^i|s_0^i, \tau_{j,<t}^i)}$ is the importance sampling ratio, and $\text{CLIP}(r, A, \epsilon) = \min[r \cdot A, \text{clip}(r; 1-\epsilon, 1+\epsilon) \cdot A]$ keeps the updated policy within a trust region of the old policy. Following recent studies~\citep{yu2025dapoopensourcellmreinforcement, liu2025understanding}, we exclude the KL penalty as it has minimal impact on performance.

\subsection{Thompson Sampling and Beta-Binomial Posterior Mean}
Thompson sampling~\citep{Sutton1998} is a Bayesian approach to the exploration-exploitation tradeoff that selects actions by sampling from the posterior distribution over each action's expected reward. We adopt its posterior model, though not its sampling rule, in order to quantify how well the current policy handles a given prompt. For each prompt $s_0^i$, define the prompt quality parameter $\alpha_{s_0^i} \in [0,1]$ as the expected reward under the current policy $\pi_{\theta}$,
$\alpha_{s_0^i} = \E_{\tau \sim \pi_\theta \mid s_0^i} [\mathcal{R}(\tau)]$, where the reward operator $\mathcal{R}(\tau)$ returns 1 if trajectory $\tau$ outputs the correct final answer and 0 otherwise. This quantity is unavailable before trajectory sampling, so we place a uniform prior $\alpha_{s_0^i} \sim \text{Beta}(1,1)$ on it.

The group of trajectories $\calG^i$ sampled for prompt $s_0^i$ constitutes $N$ Bernoulli trials, each succeeding with probability $\alpha_{s_0^i}$. The number of successes $S_i = \sum_{j=1}^N \mathcal{R}(\tau_j^i)$ therefore satisfies $S_i \sim \text{Binomial}(N, \alpha_{s_0^i})$, and Beta-Binomial conjugacy~\citep{bishop2006pattern} yields the posterior distribution and its mean,
\begin{equation}
   \alpha_{s_0^i} \mid \calG^i \sim \text{Beta}(1 + S_i, 1 + N - S_i) ,\quad c_i = \frac{1 + S_i}{2 + N}
   \label{eq:bayes_confidence_score}
\end{equation}

We define the Bayesian confidence score $c_i$ of a prompt as this posterior mean, which combines the observed success count with the uniform prior. Two properties matter for our purpose. First, the prior keeps $c_i$ strictly inside $(0,1)$, so a group with no observed success receives a small positive estimate rather than certainty of failure, which is appropriate at the group sizes used in practice ($N=8$). Second, $c_i$ takes values on a grid whose spacing is determined by $N$, whereas the raw rate $S_i/N$ assigns the same value to every failed group regardless of the group size. Section~\ref{sec:gating} exploits this dependence to set the gating threshold from $N$ alone. We use the posterior mean rather than a draw from the posterior, because the score serves as a control signal rather than an exploration device, and sampling would add variance without changing which groups are selected.

\section{Hindsight-Anchored Policy Optimization}

In this section, we present HAPO, a method-agnostic framework that augments mixed-policy post-training with three components, namely the Synthetic Success Injection (SSI) operator, Beta-Binomial Confidence Gating, and adaptive threshold annealing. We then state the overall training objective. The complete procedure is given in Algorithm~\ref{algorithm:hapo}.

\subsection{The Synthetic Success Injection (SSI) Operator}

When a group $\calG^i$ exhibits low confidence, the model's policy requires additional guidance to improve learning. To address this scenario, we define the SSI operator $\calT$ that operates at the group level. Within a low-confidence group $\calG^i$, the poorest-performing trajectory $j^*=\arg\min_{j} \mathcal{R}(\tau_j^i)$ is identified and replaced by a verified teacher trajectory $\tau_*^i$:

\begin{equation}
  \mathcal{T}(\calG^i) = \{\tau_1^i, ..., \tau_{j^*-1}^i, \tau_*^i, \tau_{j^*+1}^i, ..., \tau_N^i\}
  \label{eq:transform}
\end{equation}

The teacher trajectory acts as a hindsight anchor. It restores reward variance inside the group, so the update recovers a direction of improvement on a prompt the current policy cannot solve. Section~\ref{sec:gating} decides on which groups $\calT$ is applied.

\subsection{Beta-Binomial Confidence Gating}
\label{sec:gating}

Applying $\calT$ to every group would inject teacher guidance where the policy already succeeds. When most trajectories in $\calG^i$ receive reward 1, the on-policy gradient is informative and intervention is unwarranted. We therefore gate SSI on the Bayesian confidence score $c_i$ of Eq.~(\ref{eq:bayes_confidence_score}) and apply $\calT$ only when $c_i<\gamma_t$, where $\gamma_t$ is the gating threshold at step $t$.

While $\gamma_t$ can be set manually, its admissible range is determined by the group size $N$. Because every trajectory reward is binary, the confidence score takes values only on the discrete ladder
\begin{equation}
c_i \in \left\{ \tfrac{1}{2+N},\ \tfrac{2}{2+N},\ \dots,\ \tfrac{N+1}{2+N} \right\},
\qquad c_i = \tfrac{1+S_i}{2+N}
\label{eq:confidence-ladder}
\end{equation}
The gate $c_i<\gamma_t$ therefore depends not on the value of $\gamma_t$ but only on \emph{which two rungs it separates}, so any threshold inside a rung gap behaves identically. This reduces the choice of $\gamma_t$ to a single integer $m\ge 0$, the largest successful rollout count still treated as group failure.

\begin{proposition}[Gate calibration]
\label{prop:gate}
For any integer $m\ge 0$,
\begin{equation}
\tfrac{m+1}{2+N} < \gamma_t \le \tfrac{m+2}{2+N}
\iff
\bigl[\,\calT\ \text{applied}\,\bigr]\Leftrightarrow\bigl[\,S_i\le m\,\bigr].
\label{eq:all-fail-band}
\end{equation}
\end{proposition}
\begin{proof}
$c_i<\gamma_t$ reads $1+S_i<(2+N)\gamma_t$. Under the stated band this is
$1+S_i<m+2$ at the upper end and $1+S_i\le m+1$ at the lower; since $S_i$ is an
integer, both are equivalent to $S_i\le m$.
\end{proof}

Taking the midpoint of the corresponding gap gives a tuning-free threshold which adjusts with $N$.
\begin{equation}
\gamma_0 \;=\; \frac{m+\tfrac{3}{2}}{2+N},
\label{eq:gamma0-canonical}
\end{equation}
Our default is $m=0$, injecting only when the entire group fails ($S_i=0$). Larger $m$ is useful when the verifier is noisy or when a single successful rollout is weak evidence of competence. The same formulation covers both cases without further modification.

\begin{algorithm}[t]
\caption{Hindsight-Anchored Policy Optimization}
\label{algorithm:hapo}
\textbf{Input}: Dataset $\calD=\{(s_0^i,\tau_\ast^i)\}$; policy $\pi_\theta$; group size $N$; groups per step $M$; steps per epoch $T_{\mathrm{epoch}}$; turn-off epoch $e_{\mathrm{off}}$; steepness $k$; failure level $m$ ($m{=}0$ by default); learning rate $\alpha$\\
\textbf{Output}: Optimized policy $\pi_\theta$
\begin{algorithmic}[1]
\STATE $\gamma_0 \leftarrow \tfrac{m+3/2}{2+N}$,\quad
       $\gamma_\infty \leftarrow \tfrac{1}{2(2+N)}$
       \COMMENT{gate endpoints derived from $N$}
\FOR{step $t = 1$ to $T$}
    \STATE $e_t \leftarrow t / T_{\mathrm{epoch}}$
    \STATE $\gamma_t \leftarrow \gamma_\infty + (\gamma_0-\gamma_\infty)\bigl(1-\sigma(k(e_t-e_{\mathrm{off}}))\bigr)$
    \STATE Sample a batch of prompts $\{s_0^i : i \in \{1,\dots,M\}\}$ from $\calD$
    \FOR{$i = 1$ to $M$}
        \STATE Sample $N$ trajectories $\calG^i = \{\tau_j^i \sim \pi_\theta(\cdot\mid s_0^i) : j \in \{1,\dots,N\}\}$
        \STATE Compute rewards $\calR(\tau_j^i)$ for all $\tau_j^i \in \calG^i$
        \STATE Compute success count $S_i = \sum_{j=1}^{N} \mathbb{I}[\calR(\tau_j^i) = 1]$
        \STATE Compute confidence score $c_i = \frac{1+S_i}{2+N}$
        \IF{$c_i < \gamma_t$}
            \STATE $\widetilde{\calG}^i \leftarrow \calT(\calG^i) = \bigl(\calG^i \setminus \{\arg\min_{\tau \in \calG^i} \calR(\tau)\}\bigr) \cup \{\tau_\ast^i\}$
            \COMMENT{low confidence: replace worst trajectory with teacher trajectory}
        \ELSE
            \STATE $\widetilde{\calG}^i \leftarrow \calG^i$
            \COMMENT{high confidence: keep the group on-policy}
        \ENDIF
    \ENDFOR
    \STATE Update policy: $\theta \leftarrow \theta + \alpha \nabla_\theta \calJ_{\mathrm{HAPO}}(\theta)$ using $\{\widetilde{\calG}^i\}_{i=1}^{M}$
\ENDFOR
\STATE \textbf{return} $\pi_\theta$
\end{algorithmic}
\end{algorithm}

\subsection{Global Curriculum Through Adaptive Threshold Annealing}
\label{sec:annealing}

On top of the adaptive gating, we add an annealing schedule as a global-horizon mechanism that stabilizes training under off-policy drift. Gating is endogenous and closes only as the policy improves, so on prompts the policy never solves it stays open for the whole run. This residual exposure carries a cost. The injected trajectory is not drawn from $\pi_\theta$, so its likelihood ratio is unbounded and its gradient share does not decay as $\pi_\theta$ drifts from the teacher. Sustaining this mixture over the global horizon leads to training collapse (Fig.~\ref{fig:stability}).

The annealing schedule operates between two conditions: the cold-start gate $\gamma_0$ of Eq.~(\ref{eq:gamma0-canonical}) that injects on failed groups, and a floor $\gamma_\infty=\tfrac{1}{2(2+N)}$ that lies below the minimum attainable confidence $\tfrac{1}{2+N}$ (Eq.~\ref{eq:confidence-ladder}), so once $\gamma_t$ reaches the floor the gate is closed for the remaining run. Let $e_t=t/T_{\mathrm{epoch}}$ be training progress in epochs, $\sigma$ the logistic sigmoid, $k$ a fixed steepness ($k=10$), and $e_{\mathrm{off}}>0$ the epoch at which intervention is switched off:
\begin{equation}
\gamma_t = \gamma_\infty + (\gamma_0-\gamma_\infty)\bigl(1-\sigma(k(e_t-e_{\mathrm{off}}))\bigr).
\label{eq:annealed-threshold}
\end{equation}

The schedule proceeds through three phases. During the warm start phase ($e_t\ll e_{\mathrm{off}}$), the threshold satisfies $\gamma_t\approx\gamma_0$ and supplies the corrective signal that the cold-start stage requires. During the smooth transition phase ($e_t\approx e_{\mathrm{off}}$), $\gamma_t$ decreases continuously across the rungs of Eq.~(\ref{eq:confidence-ladder}), so the injection criterion tightens by one success count at a time rather than in a single step. In the finite extinction phase, $\gamma_t$ falls below $\tfrac{1}{2+N}$ after a finite number of epochs, and the gate remains closed for the remainder of the run. Under the default $m=0$, the gate begins by injecting a teacher trajectory only on groups whose rollouts all fail, and ends by injecting none.

\mainresultstable

\subsection{HAPO Objective Function}

After applying the Beta-Binomial confidence gating, each group $\calG^i$ contains either only on-policy trajectories or one teacher-replaced trajectory produced by the SSI operator, denoted as $\widetilde{\calG}^i = \calT(\calG^i)$. For each trajectory $\tau^i_j \in \widetilde{\calG}^i$ we define a per-trajectory contribution. On-policy rollouts use the standard GRPO clipped surrogate objective, and the injected teacher trajectory uses a confidence-weighted teacher-forcing term. Formally, the HAPO objective is given by:

\begin{equation}
\calJ_{\mathrm{HAPO}}(\theta)
=
\frac{1}{\sum_{i=1}^M \sum_{\tau^i_j \in \widetilde{\calG}^i} |\tau^i_j|}
\sum_{i=1}^M \sum_{\tau^i_j \in \widetilde{\calG}^i}
\calL(\theta; \tau^i_j)
\label{eq:hapo-objective}
\end{equation}
where
\begin{equation}
\mathcal{L}(\theta; \tau^i_j)
=
\begin{cases}
\displaystyle
\sum_{t=1}^{|\tau^i_j|}
\mathrm{CLIP}\!\left(r^i_{j,t}(\theta), A^i_j, \epsilon\right)
& \tau^i_j \neq \tau_i^\ast, \\[1.2em]
\displaystyle
\lambda(c_i)
\sum_{t=1}^{|\tau_\ast^i|}
\calF( \pi_\theta\!\left(\tau_{\ast,t}^i \mid s_0^i, \tau_{\ast,<t}^i\right))
& \tau^i_j = \tau_\ast^i
\end{cases}
\label{eq:hapo-loss}
\end{equation}

The coefficient $\lambda(c_i)$ modulates the influence of the teacher trajectory according to the posterior confidence. Taking $\lambda$ to be a monotone decreasing function of $c_i$ makes teacher forcing strongest on low-confidence groups and weakest near the gating boundary. The simplest choice is the indicator $\lambda(c_i) = \mathbb{I}[c_i < \gamma_t]$, under which the teacher term is active whenever the gate is open. The policy shaping operator $\calF$ reshapes the distribution over tokens and determines the scaling of the gradient weights. Common operators and their gradient weights are summarized in Appendix~\ref{sec:operators}.

\section{Experiments}

In this section, we present our experimental setup, main results, and ablation study. For full experiment details and supplementary ablation studies, refer to Appendix~\ref{sec:exp_details} and \ref{sec:ablation}.

We evaluate HAPO as an extension layer for mixed-policy post-training in mathematical reasoning. Our empirical study examines:
(1) whether HAPO improves representative mixed-policy baselines;
(2) whether gated, cold-start-targeted injection outperforms always-on injection; and
(3) whether annealing is necessary for training stability.

\begin{figure*}[t]
    \centering
    \includegraphics[width=0.8\textwidth]{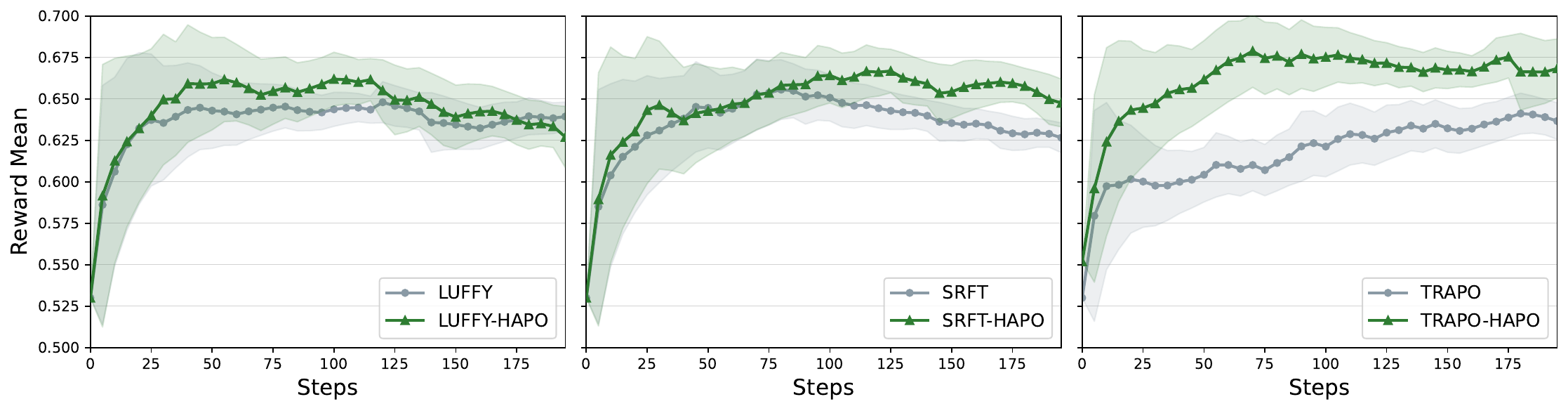}
    \caption{Validation reward dynamics comparing HAPO with baseline methods in cold-start phase. From left to right: LUFFY, SRFT, and TRAPO. Each subplot shows the average reward on the validation set across training steps, with HAPO variants (green) compared against their respective baselines (gray). Additional results can be found in Appendix~\ref{sec:ablation}.}
    \label{fig:training_dynamics}
\end{figure*}

\subsection{Experimental Setup}

\paragraph{Training Setup.} We conduct experiments on the MATH dataset~\citep{DBLP:conf/nips/HendrycksBKABTS21}, using both training and validation splits from difficulty levels 3-5 following~\citet{zeng2025simplerlzooinvestigatingtamingzero}. Teacher trajectories in training split are generated using gpt-oss-120b~\citep{openai2025gptoss120bgptoss20bmodel}. Following prior studies~\citep{zeng2025simplerlzooinvestigatingtamingzero,liu2025uftunifyingsupervisedreinforcement}, we use Qwen2.5-3B~\citep{qwen2025qwen25technicalreport} as our base model and GRPO~\citep{shao2024deepseekmath} excluding the KL penalty term~\citep{liu2025understanding} as our RL algorithm. For all RL experiments (GRPO, LUFFY, SRFT, and TRAPO), we use a batch size of 128, constant learning rate of $5 \times 10^{-6}$, trajectory generation temperature of 1.0, group size of 8, and maximum generation length of 8{,}192 tokens. All remaining hyperparameters follow established baselines~\citep{yan2025learningreasonoffpolicyguidance,fu2025srftsinglestagemethodsupervised,su2025trustregionadaptivepolicyoptimization} and are kept consistent across experiments. For HAPO variants, we maintain the same base hyperparameters but vary three components: (1) whether to apply the SSI operator $\calT$ from Eq.~(\ref{eq:transform}), (2) the gate, either the derived tuning-free $\gamma_0$ of Eq.~(\ref{eq:gamma0-canonical}) ($m=0$) or a fixed $\gamma_0\in\{0.3,0.5,0.7\}$ for the sensitivity study, and (3) whether to apply threshold annealing $\gamma_t$ from Eq.~(\ref{eq:annealed-threshold}).

\paragraph{Evaluation Setup.} For evaluation, we follow~\citet{yan2025learningreasonoffpolicyguidance} and employ the Math-verify\footnote{\url{https://github.com/huggingface/Math-Verify}} and OAT-Grader~\citep{liu2024oat} frameworks with temperature 0.6 and maximum generation length of 8{,}192 tokens. We assess our approach on six mathematical reasoning benchmarks: MATH-500~\citep{DBLP:conf/nips/HendrycksBKABTS21}, OlympiadBench~\citep{DBLP:conf/acl/HeLBHTSHHHZLQL024}, Minerva~\citep{DBLP:conf/nips/LewkowyczADDMRS22}, AIME2024, AIME2025, and AMC~\citep{li2024numinamath}. Following standard evaluation protocols, we report avg@32 for AIME 2024, AIME 2025 and AMC due to its limited test samples, while using pass@1 for the other three benchmarks.

\paragraph{Baseline Comparison.}
We compare against several baseline methods: (1) GRPO~\citep{shao2024deepseekmath}, (2) SFT, which directly trains the model to imitate teacher trajectories; (3) SFT-GRPO, following the standard two-stage pipeline where SFT precedes RL; and three recent methods that uniformly replace one trajectory per group with an teacher trajectory: (4) SRFT~\citep{fu2025srftsinglestagemethodsupervised}, which uses SFT token loss; (5) LUFFY~\citep{yan2025learningreasonoffpolicyguidance}, which incorporates policy shaping with SFT token loss; and (6) TRAPO~\citep{su2025trustregionadaptivepolicyoptimization}, which incorporates trust-region policy shaping with SFT token loss. To demonstrate the effectiveness of HAPO, we apply it on top of LUFFY, SRFT, and TRAPO, creating LUFFY-HAPO, SRFT-HAPO, and TRAPO-HAPO variants.

\subsection{Main Results}

\paragraph{Mathematical Reasoning Performance.}
Table~\ref{tab:main_results} reports results at two scales, Qwen2.5-3B and Qwen2.5-7B. At the 3B scale, TRAPO-HAPO achieves the strongest overall performance with an average score of \textbf{28.90}, and HAPO improves every base method: LUFFY-HAPO over LUFFY by \textbf{+1.51}, SRFT-HAPO over SRFT by \textbf{+1.08}, and TRAPO-HAPO over TRAPO by \textbf{+4.03}. The same pattern holds at the 7B scale, where HAPO improves all three base methods (LUFFY \textbf{+0.87}, SRFT \textbf{+4.52}, TRAPO \textbf{+2.71}), indicating that the benefits of adaptive teacher intervention do not vanish with a stronger base model. These results confirm that HAPO's adaptive integration of teacher trajectories leads to more effective reasoning skill acquisition than uniform application strategies.

\paragraph{Training Dynamics.}
Figure~\ref{fig:training_dynamics} shows that HAPO variants consistently achieve higher reward means than their respective baselines in cold-start phase. Notably, the performance gap between HAPO and baseline methods widens over the course of training, suggesting that HAPO's adaptive integration of teacher trajectories becomes increasingly beneficial as the model learns. This trend is most pronounced in the TRAPO-HAPO setting, as is also reflected in Table~\ref{tab:main_results}.

\ablationtable

\subsection{Ablation Study}

Table~\ref{tab:ablation} decomposes HAPO into its three components, using TRAPO as the base method and restricting training to the cold-start stage ($<200$ steps). The ablation is cumulative. The row \textit{+ Hindsight} applies the SSI operator of Eq.~(\ref{eq:transform}) to every group, without gating. The row \textit{+ Gating} restricts injection to groups whose confidence score falls below the threshold of Eq.~(\ref{eq:gamma0-canonical}). The row \textit{+ Annealing} adds the schedule of Eq.~(\ref{eq:annealed-threshold}) and yields the full framework.

Two observations follow. First, SSI alone does not improve on the base method, reducing average accuracy from 24.87 to 23.80. Injecting a teacher trajectory into every group therefore supplies guidance on prompts the policy already solves and the replacement discards a useful rollout. Second, gating recovers this loss and improves substantially on both the base method and ungated injection, reaching 28.90 and attaining the highest score on five of the six benchmarks.

Annealing leaves average accuracy essentially unchanged within this budget, at 27.64 against 28.90. Over the cold-start stage the gated runs have not accumulated enough off-policy drift for the schedule to affect benchmark scores, and accuracy is not the quantity annealing controls. Its role is to bound the off-policy horizon and preserve stability over the global horizon, a discussion we isolate in Section~\ref{sec:stability}.

\subsection{Annealing Mitigates Off-Policy Collapse}
\label{sec:stability}

Gating determines which groups receive off-policy gradient, but not for how long. On prompts the policy never solves, the gate remains open for the entire run. Fig.~\ref{fig:stability} reports the cost of this residual exposure when training is extended over the global horizon without annealing. Policy entropy and the KL divergence to the initial policy grow monotonically rather than plateauing, and beyond the point at which they diverge, training collapses and does not recover. This is the mechanism identified in Section~\ref{sec:annealing}. With annealing, the injected fraction is driven to zero at $e_{\mathrm{off}}$, entropy and KL divergence stabilize, and reward continues to improve. We therefore regard annealing as an essential stability measure for mixed-policy training, rather than as a component for improving accuracy.

\begin{figure}[!ht]
    \centering
    \includegraphics[width=\columnwidth]{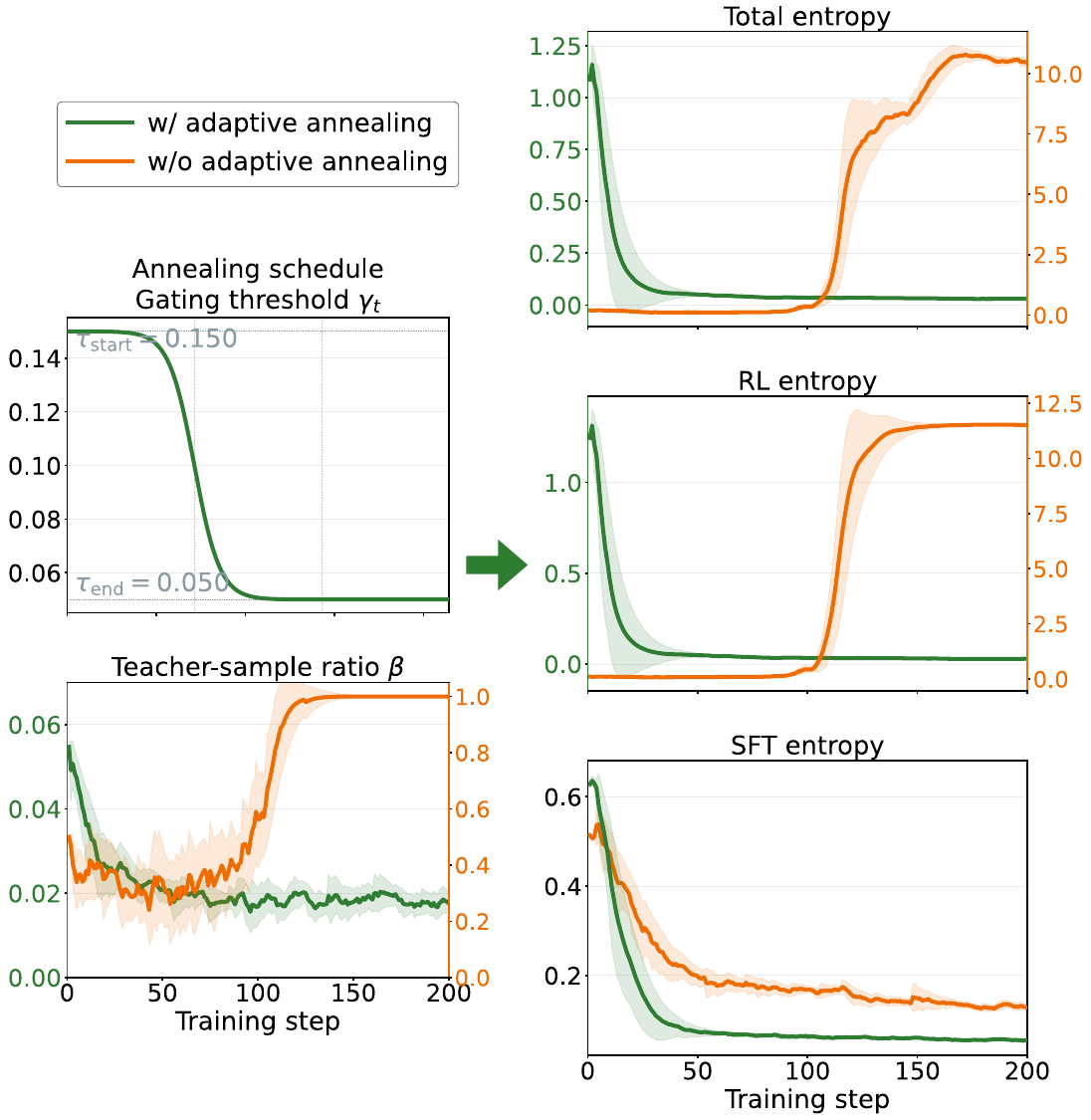}
    \caption{Adaptive annealing prevents training collapse.
  \textbf{Left:} The gating threshold $\gamma_t$ decays sigmoidally from $\tau_{\mathrm{start}}$ to $\tau_{\mathrm{end}}$, progressively withdrawing teacher support. Accordingly, the teacher-sample fraction $\beta$ remains low under adaptive annealing (green), but rises toward one without annealing as a sign of potential training collapse (orange).
  \textbf{Right:} Without annealing (orange), total and RL entropy explode after $\sim 100$ steps. With annealing (green), all entropy components stabilize as teacher intervention is progressively withdrawn.}
  \label{fig:stability}

\end{figure}

\section{Conclusion}
In this paper, we propose HAPO, a mixed-policy training framework that dynamically adapts teacher intervention according to the policy’s current competence, enabling teacher guidance to serve as temporary and self-retracting support during learning. Through extensive experiments, we demonstrate that teacher guidance can effectively improve learning in sparse-reward settings without becoming a persistent component throughout the entire training process. Furthermore, we show that HAPO can be integrated with existing mixed-policy approaches, improving their average performance while maintaining training stability.
For future work, we plan to evaluate HAPO with larger-scale models and datasets, including more diverse reasoning benchmarks, to further assess its scalability and generalization capability. We also aim to investigate the relationship between mixed-policy training and on-policy distillation, exploring whether adaptive teacher intervention can provide complementary benefits in distillation-based training.

\paragraph{Reproducibility.}The six public datasets are standard and cited; SSI (§4.1), Beta-Binomial Confidence Gating(§4.2), the adaptive annealing objective (§4.4) are provided. Theoretical details are included in the Appendix.

\paragraph{Limitations.} Our evaluation is confined to mathematical reasoning with verifier-based binary rewards and GRPO-style group estimators, on which the gate as formulated relies. Results are from single runs per configuration, so per-benchmark differences of about a point should be read as noise, and HAPO does regress on individual benchmarks even where the average improves. Finally, $e_{\mathrm{off}}$ trades cold-start support against off-policy exposure and remains epoch-dependent, deriving it from observed training signal is a natural next step.

\bibliography{aaai2027}
\clearpage
\appendix
\onecolumn

\section{Theoretical Analysis}
\label{sec:exp_theories}

This appendix supplies the material that Section~4 refers to without developing. We first decompose the HAPO update into an on-policy component and a teacher component, and define the off-policy mass of a training step. We then establish the two pathways through which this mass is withdrawn, one driven by the competence of the policy and one enforced by the annealing schedule. We close with a comparison against fixed-schedule mixing.

\paragraph{Gradient decomposition.}
Let $I_i^t = \mathbb{I}[c_i < \gamma_t]$ denote the injection indicator for group $i$ at step $t$. Up to the token normalizer, the update induced by Eq.~(\ref{eq:hapo-objective}) is
\begin{equation}
\widehat{g}_t(\theta) = \frac{1}{M}\sum_{i=1}^M \bigl(I_i^t\,\widehat{g}_{i,\mathrm{teach}}(\theta)
+(1-I_i^t)\,\widehat{g}_{i,\mathrm{RL}}(\theta)\bigr),
\label{eq:hapo-gradient-estimator}
\end{equation}
where $\widehat{g}_{i,\mathrm{teach}}$ is the shaped teacher term of Eq.~(\ref{eq:hapo-loss}) and $\widehat{g}_{i,\mathrm{RL}}$ is the on-policy GRPO term. We define the off-policy mass of a step as $\beta_t=\tfrac{1}{M}\sum_{i=1}^M I_i^t$, the fraction of injected groups in that step.

\paragraph{Exogenous withdrawal through annealing.}
The annealing schedule of Eq.~(\ref{eq:annealed-threshold}) removes the teacher component after a finite number of epochs, determined in closed form by the group size $N$, the steepness $k$, the failure level $m$, and the turn-off epoch $e_{\mathrm{off}}$.

\begin{proposition}[Finite extinction]
\label{prop:extinct}
Under Eq.~(\ref{eq:annealed-threshold}) with $\gamma_0$ given by Eq.~(\ref{eq:gamma0-canonical}), define $e^\star = e_{\mathrm{off}} + \tfrac{1}{k}\log(2m+1)$. For all $e_t > e^\star$ we have $\gamma_t < \tfrac{1}{2+N} \le c_i$ for every group $i$, hence $I_i^t = 0$ and $\widehat{g}_t(\theta) = \tfrac{1}{M}\sum_{i=1}^M \widehat{g}_{i,\mathrm{RL}}(\theta)$.
\end{proposition}

\begin{proof}
Write $z = k(e_t - e_{\mathrm{off}})$ and use $1-\sigma(z)=\sigma(-z)$. The inequality $\gamma_t < \tfrac{1}{2+N}$ is equivalent to $(\gamma_0-\gamma_\infty)\,\sigma(-z) < \tfrac{1}{2+N}-\gamma_\infty$. Substituting $\gamma_0-\gamma_\infty = \tfrac{m+1}{2+N}$ and $\tfrac{1}{2+N}-\gamma_\infty = \tfrac{1}{2(2+N)}$ gives $\sigma(-z) < \tfrac{1}{2m+2}$, which holds if and only if $z > \log(2m+1)$, that is, $e_t > e^\star$. The bound $c_i \ge \tfrac{1}{2+N}$ is the smallest rung of Eq.~(\ref{eq:confidence-ladder}).
\end{proof}

Under the default $m=0$ we have $e^\star = e_{\mathrm{off}}$. The conclusion holds for any prompt distribution and does not require the policy to solve any particular prompt, so the total off-policy exposure of a run is bounded by construction. For $e_t > e^\star$ the objective of Eq.~(\ref{eq:hapo-objective}) coincides with the on-policy objective of Eq.~(\ref{eq:grpo}).

\paragraph{Endogenous withdrawal through gating.}
The gate also closes on its own as the policy improves, although only in the probabilistic sense stated below. Fix the threshold at $\gamma_t\equiv\gamma$ with $\gamma\le\tfrac{1}{2}$, which the thresholds of Eq.~(\ref{eq:gamma0-canonical}) satisfy for the group sizes considered here, and let $\mu$ denote the success probability of the policy on a given prompt.

\begin{proposition}[Endogenous decay]
\label{prop:endogenous}
If $\mu>\gamma$, the probability that the gate opens on this prompt satisfies
\begin{equation}
\mathbb{P}(c<\gamma)\le\exp\bigl(-2N(\mu-\gamma)^2\bigr).
\label{eq:endogenous-vanishing}
\end{equation}
\end{proposition}

\begin{proof}
The success count satisfies $S\sim\mathrm{Bin}(N,\mu)$, and $c=(1+S)/(2+N)$ is increasing in $S$, so the event $\{c<\gamma\}$ coincides with $\{S<(2+N)\gamma-1\}$. Writing $(2+N)\gamma-1 = N\mu - t$ with $t = N(\mu-\gamma)+(1-2\gamma)$, the assumptions $\mu>\gamma$ and $\gamma\le\tfrac{1}{2}$ give $t\ge N(\mu-\gamma)>0$. Hoeffding's inequality applied to the lower tail of $S$ then yields $\mathbb{P}(S\le N\mu-t)\le\exp(-2t^2/N)\le\exp\bigl(-2N(\mu-\gamma)^2\bigr)$.
\end{proof}

Injection therefore becomes exponentially rare on a prompt once the success probability of the policy exceeds the threshold, which accounts for the decay of $\beta_t$ observed in Fig.~\ref{fig:ablation-merged}. Eq.~(\ref{eq:endogenous-vanishing}) describes a decay rather than a termination, since prompts with $\mu\le\gamma$ are not excluded by the gate alone. Proposition~\ref{prop:extinct} supplies the termination that this bound cannot.

\paragraph{Contrast with a fixed schedule.}
A fixed-schedule mixed-policy method optimizes $J_{\mathrm{RL}}+\lambda J_{\mathrm{SFT}}$ with $\lambda>0$ held constant, so its off-policy mass satisfies $\beta_t\equiv 1$ for the entire run. Whenever the teacher differs from the optimum of $J_{\mathrm{RL}}$, the stationary points of this objective satisfy $\nabla J_{\mathrm{RL}}=-\lambda\nabla J_{\mathrm{SFT}}$ rather than $\nabla J_{\mathrm{RL}}=0$. Under HAPO, $\beta_t$ reaches zero within a finite and pre-specified number of epochs, and the two objectives coincide for $e_t>e^\star$. We state this as a property of the design.

\clearpage
\section{Policy Shaping Operators}
\label{sec:operators}

This section details the policy shaping operators $\calF$ used by the three mixed-policy base methods that HAPO augments in our experiments. Table~\ref{tab:gradient-comparison} summarizes each operator and its associated gradient weight, and Fig.~\ref{fig:gradient-comparison} visualizes how these gradient magnitudes depend on the policy probability $\pi_\theta$.

\begin{figure}[h!]
    \centering
    \includegraphics[width=0.45\columnwidth]{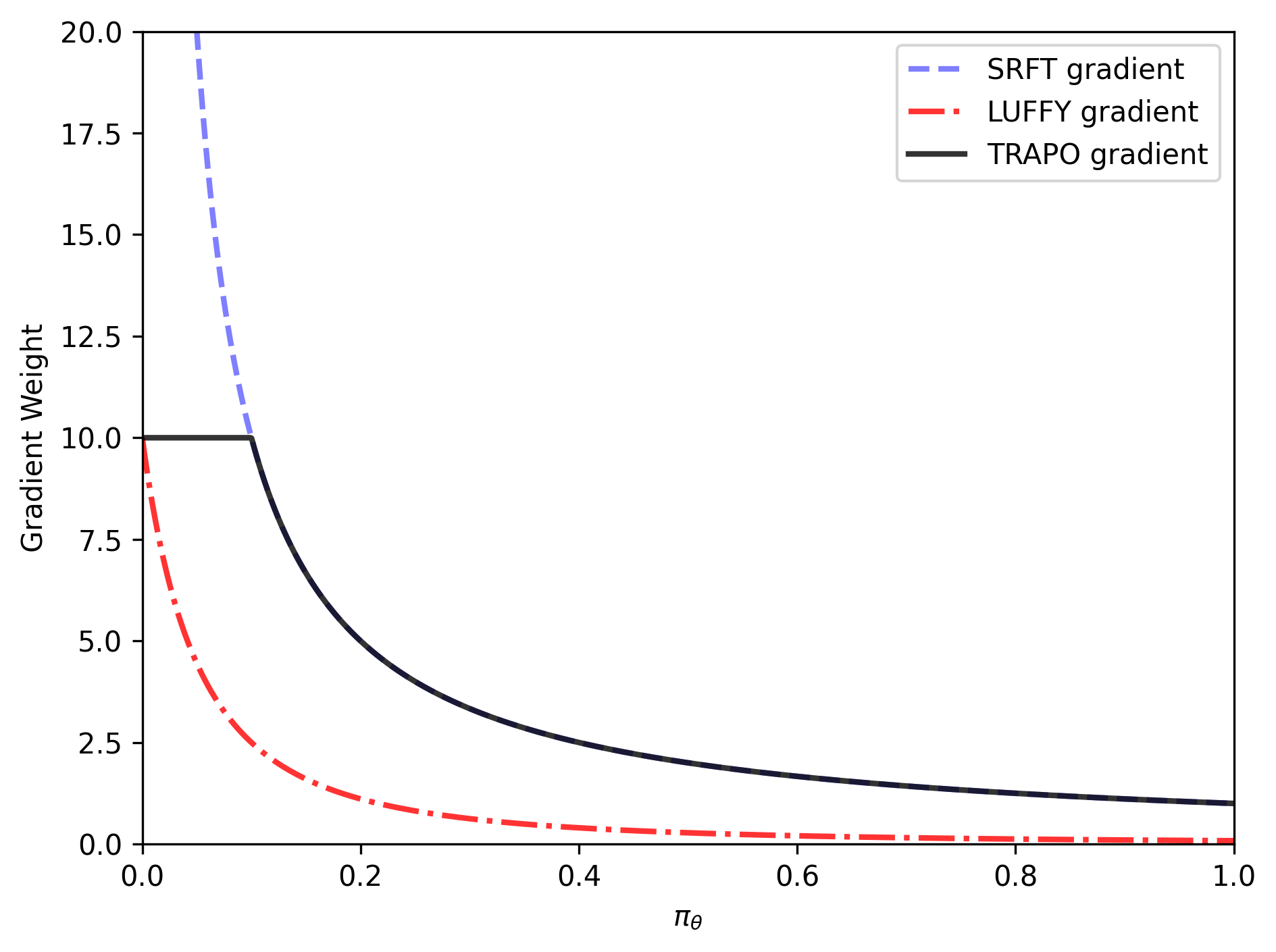}
    \caption{Comparison of gradient magnitudes for LUFFY, SRFT, and TRAPO methods as a function of policy probability $\pi_\theta$. SRFT uses inverse probability weighting, LUFFY applies smoothing to prevent gradient explosion, and TRAPO uses a threshold-based approach.}
    \label{fig:gradient-comparison}
\end{figure}

\gradientcomparisontable

\section{Experimental Details}
\label{sec:exp_details}

In this section, we provide training details not mentioned in the main text.

\paragraph{Base Model.}
Our main experiments are conducted using Qwen2.5-3B~\citep{qwen2025qwen25technicalreport}, a general-purpose model.

\paragraph{Dataset.}
Our data source is the MATH dataset~\citep{DBLP:conf/nips/HendrycksBKABTS21}, specifically problems from difficulty levels 3-5 following~\citet{zeng2025simplerlzooinvestigatingtamingzero}. Teacher trajectories are generated using gpt-oss-120b~\citep{openai2025gptoss120bgptoss20bmodel}. We filter the dataset to retain only problems for which the teacher model generates a correct answer, resulting in 8k prompts. The generation prompt follows the DeepSeek-R1~\citep{shao2024deepseekmath} style without explicit \texttt{<think>} and \texttt{</think>} token delimiters. The training prompt follows the format from~\citet{zeng2025simplerl}. The prompt templates used for generation and training are presented below.

\paragraph{Prompt Templates.} Here we show prompt templates for teacher trajectory generation and training as below. 

\begin{tcolorbox}[
    arc=0mm,
    boxrule=1pt,
    colback=blue!6!white,
    colframe=black,
    colbacktitle=black,
    title=Prompt Template for Teacher Trajectory Generation,
    boxed title style={boxrule=0pt,colframe=white}
]
You will be given a problem and its answer. Generate step-by-step reasoning that leads to this answer.

Put your final answer within \textbackslash boxed\{\}. The answer in \textbackslash boxed\{\} must be exactly: \{answer\}

Problem: \{problem\}

Provide your step-by-step solution:
\end{tcolorbox}

\begin{tcolorbox}[
    arc=0mm,
    boxrule=1pt,
    colback=blue!6!white,
    colframe=black,
    colbacktitle=black,
    title=Prompt Template for Training,
    boxed title style={boxrule=0pt,colframe=white}
]
\texttt{<|}im\_start\texttt{|>}system \\
You are a helpful assistant. \\
\texttt{<|}im\_end\texttt{|>} \\

\texttt{<|}im\_start\texttt{|>}user \\
\{problem\}

Please reason step by step, and put your final answer within \textbackslash boxed\{\}. \\
\texttt{<|}im\_end\texttt{|>} \\

\texttt{<|}im\_start\texttt{|>}assistant \\
\{output\}
\end{tcolorbox}

\paragraph{SFT.}
For all SFT methods, we train on 8k prompt-trajectory pairs for 2 epochs. We use a maximum response length of 8{,}192, learning rate of $5 \times 10^{-5}$ with a warmup ratio of 0.1, and batch size of 32. The remaining hyperparameters follow~\citet{yan2025learningreasonoffpolicyguidance}. All SFT training is conducted using the VERL framework~\citep{Sheng_2025}.

\paragraph{GRPO.}
    All GRPO-based methods use Dr.GRPO~\citep{liu2025understanding}, a variant that removes length normalization and standard deviation regularization in advantage calculation. We use a batch size of 128, 8 rollouts per group, maximum prompt length of 1{,}024, maximum response length of 8{,}192, and a constant learning rate of $5 \times 10^{-6}$. Rollouts during training are generated using vLLM~\citep{kwon2023efficientmemorymanagementlarge} with sampling temperature 1.0. All RL training is conducted using the VERL framework~\citep{Sheng_2025}.

\paragraph{SRFT.} For SRFT, we follow the original implementation~\citep{fu2025srftsinglestagemethodsupervised} but set the entropy-adaptive weighting mechanism to 0 for fair comparison, using policy shaping operator $\calF_{SRFT}$ for the SFT token loss on teacher trajectories.

\paragraph{LUFFY.} For LUFFY, we follow the original implementation~\citep{yan2025learningreasonoffpolicyguidance} using the policy shaping operator $\calF_{LUFFY}$ for the SFT token loss on teacher trajectories.

\paragraph{TRAPO.} For TRAPO, we follow the original implementation~\citep{su2025trustregionadaptivepolicyoptimization} but disable prefix selection from offline teacher trajectories for fair comparison, using policy shaping operator $\calF_{TRAPO}$ for the SFT token loss on teacher trajectories.

\paragraph{HAPO.} HAPO is applied on top of the base SRFT, LUFFY, and TRAPO methods. We maintain the same base hyperparameters but vary three components: (1) whether to apply the SSI operator $\mathcal{T}$ from Eq.~(\ref{eq:transform}), (2) the initial threshold value $\gamma_0 \in \{0.3, 0.5, 0.7\}$, and (3) whether to apply threshold annealing $\gamma_t$ from Eq.~(\ref{eq:annealed-threshold}).

\paragraph{Evaluation.}
For evaluation, we follow~\citet{yan2025learningreasonoffpolicyguidance} and employ the Math-verify\footnote{\url{https://github.com/huggingface/Math-Verify}} and OAT-Grader~\citep{liu2024oat} frameworks with temperature 0.6 and maximum generation length of 8{,}192 tokens. We assess our approach on six mathematical reasoning benchmarks: MATH-500~\citep{DBLP:conf/nips/HendrycksBKABTS21}, OlympiadBench~\citep{DBLP:conf/acl/HeLBHTSHHHZLQL024}, Minerva~\citep{DBLP:conf/nips/LewkowyczADDMRS22}, AIME2024, AIME2025, and AMC~\citep{li2024numinamath}. Following standard evaluation protocols, we report avg@32 for AIME 2024, AIME 2025 and AMC due to its limited test samples, while using pass@1 for the other three benchmarks.

\section{Hyperparameter Study}
\label{sec:ablation}

In this section, we study the effect of the initial gating threshold $\gamma_0$ on HAPO training dynamics. Since HAPO is established as the best configuration in Table~\ref{tab:ablation}, we fix that configuration and vary $\gamma_0 \in \{0.3, 0.5, 0.7\}$. Figure~\ref{fig:ablation-merged} reports three complementary signals across all three base methods: the teacher-injection fraction $\beta$, validation reward, and response length.

\paragraph{Teacher intervention retracts during training.} 
Fig.~\ref{fig:ablation-merged} (top) shows that $\beta$ decays from $\sim 0.7$--$0.9$ at initialization to $\sim 0.2$--$0.4$ by step 150 in every configuration, with smaller $\gamma_0$ retracting fastest. This is the endogenous decay of the off-policy mass $\beta_t$ analyzed in Appendix~\ref{sec:exp_theories}; the finite-extinction property (Proposition~\ref{prop:extinct}) then drives it to zero.

\paragraph{HAPO with $\gamma_0=0.3$ achieves the best performance.}
Fig.~\ref{fig:ablation-merged} (middle) shows that among HAPO variants, $\gamma_0=0.3$ (darkest green) tracks or exceeds all other settings and the base method across every panel. The largest gap appears under TRAPO with Hindsight and Thompson-Sampling gating at $\gamma_0=0.3$, consistent with the $+4.03$ average accuracy improvement reported in Table~\ref{tab:main_results}.

\paragraph{Reward gains are not driven by longer generations.} 
Fig.~\ref{fig:ablation-merged} (bottom) shows that response length remains within roughly $\pm 10\%$ of the base method across all HAPO configurations, indicating that the reward gains in Fig.~\ref{fig:ablation-merged} are not attributable to longer generations.

\paragraph{Teacher intervention retracts during training.} 
Fig.~\ref{fig:ablation-merged} (top) shows that $\beta$ decays from $\sim 0.7$--$0.9$ at initialization to $\sim 0.2$--$0.4$ by step 150 in every configuration, with smaller $\gamma_0$ retracting fastest. This is the endogenous decay of the off-policy mass $\beta_t$ analyzed in Appendix~\ref{sec:exp_theories}; the finite-extinction property (Proposition~\ref{prop:extinct}) then drives it to zero.

\begin{figure}[tp]
      \centering
      \includegraphics[width=\columnwidth]{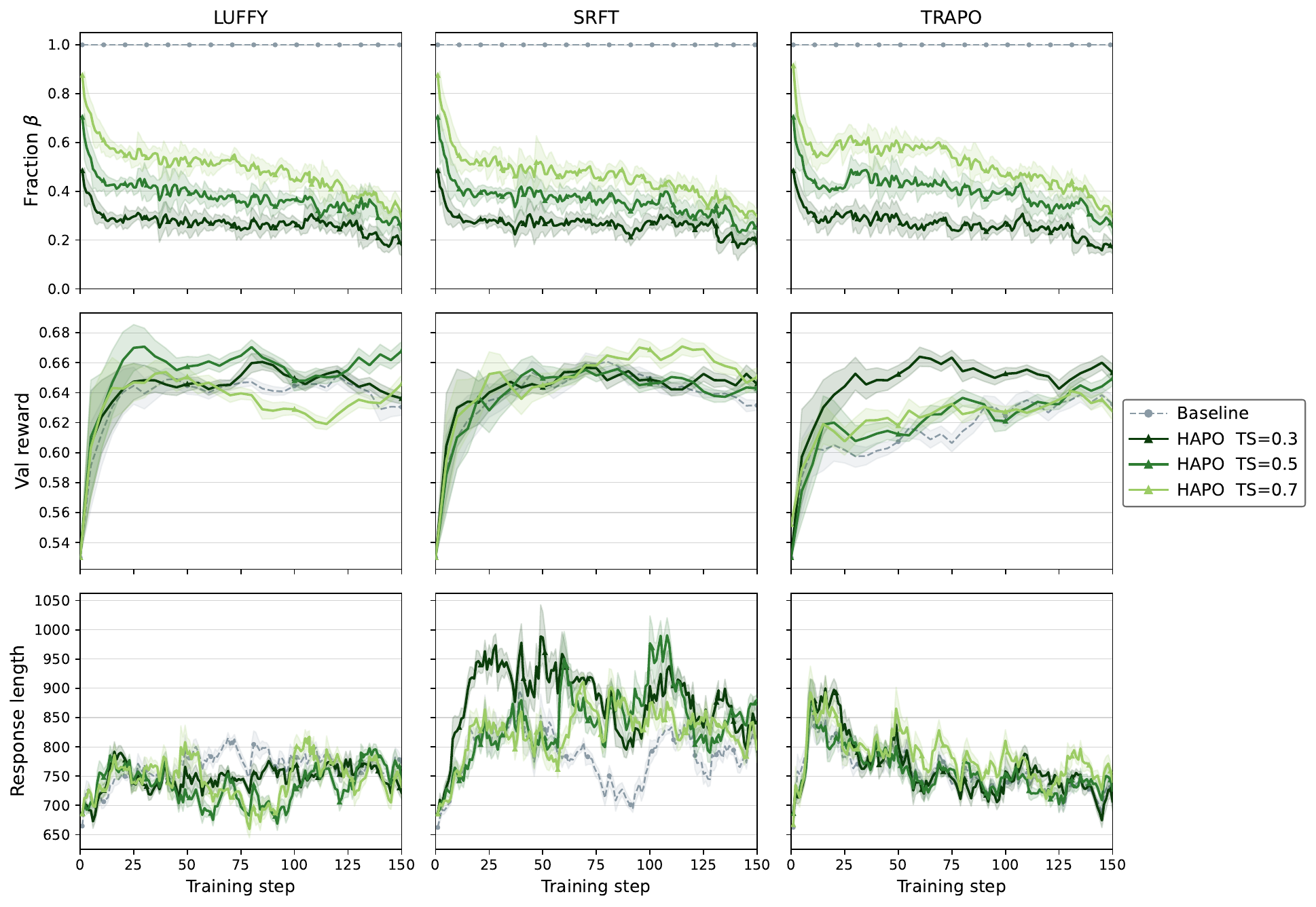}
      \caption{Training dynamics of adding HAPO under varying initial gating thresholds $\gamma_0 \in \{0.3, 0.5, 0.7\}$ (darker green = smaller $\gamma_0$), across LUFFY, SRFT, and TRAPO base methods (columns).
      \textbf{Top:} Fraction $\beta$ of teacher-injected trajectories per step; $\beta$ decays from ${\sim}0.7$--$0.9$ at initialization to ${\sim}0.2$--$0.4$ by step
      150 across all settings, with smaller $\gamma_0$ retracting fastest, empirically confirming the vanishing-intervention property.
      \textbf{Middle:} Validation reward; HAPO variants (solid green) track or exceed the base method (gray dashed), with $\gamma_0{=}0.3$ achieving the largest gains.
      \textbf{Bottom:} Response length remains within roughly ${\pm}10\%$ of the base method across all configurations, indicating that reward gains are not attributable
      to longer generations.}
  \label{fig:ablation-merged}
  \end{figure}

\end{document}

%% file: tables.tex
\newcommand{\mainresultstable}{%
\begin{table*}[!t]
\centering
\small
\setlength{\tabcolsep}{3pt}
\begin{tabular}{lccccccc}
\toprule
Method & Math-500 & Olympiad Bench & Minerva & AIME24 & AMC & AIME25 & Avg Acc \\
\midrule
\multicolumn{8}{l}{\textit{Base model: Qwen2.5-3B}} \\
\midrule
Qwen2.5-3B  & 24.40\nodelta & 11.11\nodelta & 11.76\nodelta & 3.13\nodelta & 14.50\nodelta & 1.56\nodelta & 11.08\nodelta \\
GRPO        & 63.80\nodelta & 25.78\nodelta & 27.21\nodelta & 4.27\nodelta & 31.25\nodelta & 1.67\nodelta & 25.66\nodelta \\
SFT         & 44.40\nodelta & 16.44\nodelta & 11.76\nodelta & 1.25\nodelta & 18.90\nodelta & 0.83\nodelta & 15.59\nodelta \\
SFT-GRPO    & 56.20\nodelta & 23.56\nodelta & 25.37\nodelta & 1.46\nodelta & 28.99\nodelta & 1.67\nodelta & 22.87\nodelta \\
\midrule
LUFFY       & 63.60\nodelta & 26.81\nodelta & 26.80\nodelta & 5.62\nodelta & 32.34\nodelta & 1.67\nodelta & 26.15\nodelta \\
LUFFY-HAPO  & \ttwo{66.60}\deltapct{+3.00} & 28.15\deltapct{+1.34} & 27.94\deltapct{+1.14} & \tone{7.08}\deltapct{+1.46} & 33.28\deltapct{+0.94} & 2.92\deltapct{+1.25} & \ttwo{27.66}\deltapct{+1.51} \\
SRFT        & 61.80\nodelta & 28.44\nodelta & 25.37\nodelta & \ttwo{5.94}\nodelta & 30.50\nodelta & \tone{3.23}\nodelta & 25.88\nodelta \\
SRFT-HAPO   & 64.20\deltapct{+2.40} & \ttwo{28.59}\deltapct{+0.15} & \ttwo{28.70}\deltapct{+3.33} & 4.17\deltapct{-1.77} & \ttwo{33.43}\deltapct{+2.93} & 2.71\deltapct{-0.52} & 26.96\deltapct{+1.08} \\
TRAPO       & 60.60\nodelta & 26.81\nodelta & 24.63\nodelta & 4.17\nodelta & 30.31\nodelta & 2.71\nodelta & 24.87\nodelta \\
TRAPO-HAPO  & \tone{67.00}\deltapct{+6.40} & \tone{31.56}\deltapct{+4.75} & \tone{30.88}\deltapct{+6.25} & 5.62\deltapct{+1.45} & \tone{35.32}\deltapct{+5.01} & \ttwo{3.02}\deltapct{+0.31} & \tone{28.90}\deltapct{+4.03} \\
\midrule
\multicolumn{8}{l}{\textit{Base model: Qwen2.5-7B}} \\
\midrule
Qwen2.5-7B  & 28.40\nodelta & 16.00\nodelta & 5.88\nodelta & 4.90\nodelta & 19.09\nodelta & 1.25\nodelta & 12.59\nodelta \\
GRPO        & 73.60\nodelta & 38.07\nodelta & 29.04\nodelta & 18.44\nodelta & 47.82\nodelta & 5.31\nodelta & 35.38\nodelta \\
SFT         & 45.80\nodelta & 17.93\nodelta & 18.38\nodelta & 3.44\nodelta & 20.44\nodelta & 1.04\nodelta & 17.84\nodelta \\
SFT-GRPO    & 53.60\nodelta & 23.70\nodelta & 23.53\nodelta & 4.06\nodelta & 25.56\nodelta & 3.13\nodelta & 22.26\nodelta \\
\midrule
LUFFY       & \ttwo{76.20}\nodelta & 38.37\nodelta & \tone{34.56}\nodelta & 12.40\nodelta & 46.08\nodelta & 6.67\nodelta & 35.71\nodelta \\
LUFFY-HAPO  & \ 75.80\deltapct{-0.40} & \ttwo{40.74}\deltapct{+2.37} & \ttwo{33.82}\deltapct{-0.74} & 13.23\deltapct{+0.83} & \ttwo{46.95}\deltapct{+0.87} & \ttwo{8.96}\deltapct{+2.29} & \ttwo{36.58}\deltapct{+0.87} \\
SRFT        & 70.60\nodelta & 33.93\nodelta & 32.72\nodelta & \ttwo{13.96}\nodelta & 43.49\nodelta & 7.50\nodelta & 33.70\nodelta \\
SRFT-HAPO   & \tone{77.20}\deltapct{+6.60} & 39.70\deltapct{+5.77} & \tone{34.56}\deltapct{+1.84} & \tone{15.21}\deltapct{+1.25} & \tone{48.72}\deltapct{+5.23} & \tone{13.96}\deltapct{+6.46} & \tone{38.22}\deltapct{+4.52} \\
TRAPO       & 73.20\nodelta & 35.56\nodelta & 29.78\nodelta & 10.42\nodelta & 42.85\nodelta & \ 8.44\nodelta & 33.37\nodelta \\
TRAPO-HAPO  & 74.00\deltapct{+0.80} & \tone{41.93}\deltapct{+6.37} & \tone{34.56}\deltapct{+4.78} & 12.71\deltapct{+2.29} & 45.37\deltapct{+2.52} & 7.92\deltapct{-0.52} & 36.08\deltapct{+2.71} \\
\bottomrule
\end{tabular}
\caption{Main results on mathematical reasoning benchmarks. HAPO is evaluated as an add-on regulator layer on top of three representative mixed-policy post-training methods (LUFFY, SRFT, TRAPO) and across two base-model scales (Qwen2.5-3B and Qwen2.5-7B). At both scales and across all three base methods, HAPO-augmented variants improve average accuracy over their corresponding baselines. Within each base-model block and column, the \colorbox{topone}{\textbf{top-1}} and \colorbox{toptwo}{top-2} results are highlighted; the parenthetical after each HAPO-augmented value is the absolute delta versus its base method.}
\label{tab:main_results}
\end{table*}%
}

\newcommand{\ablationtable}{%
\begin{table*}[t]
\centering
\setlength{\tabcolsep}{3pt}
\small
\begin{tabular}{lccccccc}
\toprule
 & Math-500 & Olympiad Bench & Minerva & AIME24 & AMC & AIME25 & Avg Acc \\
\midrule
TRAPO & 60.60 & 26.81 & 24.63 & 4.17 & 30.31 & 2.71 & 24.87 \\
+ Hindsight & 54.40 & 26.07 & \ttwo{27.94} & 3.33 & 28.84 & 2.19 & 23.80 \\
+ Hindsight + Gating & \tone{67.00} & \tone{31.56} & \tone{30.88} & \tone{5.62} & \tone{35.32} & \ttwo{3.02} & \tone{28.90} \\
+ Hindsight + Gating + Annealing & \ttwo{66.80} & \ttwo{29.63} & 27.57 & \ttwo{5.31} & \ttwo{33.09} & \tone{3.44} & \ttwo{27.64} \\
\bottomrule
\end{tabular}
\caption{Ablation study of HAPO components in cold-start phase. ``+ Hindsight'' corresponds to Synthetic Success Injection (SSI), ``+ Gating'' adds Beta-Binomial confidence gating, and ``+ Annealing'' yields the full HAPO framework. Within each column, the \colorbox{topone}{\textbf{top-1}} and \colorbox{toptwo}{top-2} results are highlighted.}
\label{tab:ablation}
\end{table*}%
}

\newcommand{\gradientcomparisontable}{%
\begin{table}[h!]
\centering
\small
\begin{tabular}{lcc}
\toprule
Method & Operator $\calF$ & Gradient Weight \\
\midrule
LUFFY & $\calF_{LUFFY}(\pi_\theta) = \frac{d}{\pi_\theta + d}$ & $|\nabla_{\theta}\calF_{LUFFY}| = \frac{d}{(\pi_\theta + d)^2}$ \\
SRFT & $\calF_{SRFT}(\pi_\theta) = \log{\pi_\theta}$ & $|\nabla_{\theta}\calF_{SRFT}| = \frac{1}{\pi_\theta}$ \\
TRAPO & $\calF_{TRAPO}(\pi_\theta) = \min\!\left(\frac{1}{\pi_\theta}, \frac{1}{d}\right)$ & $|\nabla_{\theta}\calF_{TRAPO}| = \begin{cases} \frac{1}{d} & \pi_\theta < d \\ \frac{1}{\pi_\theta} & \text{o.w.} \end{cases}$ \\
\bottomrule
\end{tabular}
\caption{Comparison of policy shaping operator $\calF$ and corresponding gradient weights, where $\pi_\theta$ is the policy probability and $d=0.1$ is a hyperparameter for smoothing~\citep{yan2025learningreasonoffpolicyguidance} and thresholding~\citep{su2025trustregionadaptivepolicyoptimization}.}
\label{tab:gradient-comparison}
\end{table}%
}

%% file: aaai2027.bib
@misc{zhang2021hindsighttrustregionpolicy,
      title={Hindsight Trust Region Policy Optimization}, 
      author={Hanbo Zhang and Site Bai and Xuguang Lan and David Hsu and Nanning Zheng},
      year={2021},
      eprint={1907.12439},
      archivePrefix={arXiv},
      primaryClass={cs.LG},
      url={https://arxiv.org/abs/1907.12439}, 
}

@misc{kwon2023efficientmemorymanagementlarge,
      title={Efficient Memory Management for Large Language Model Serving with PagedAttention}, 
      author={Woosuk Kwon and Zhuohan Li and Siyuan Zhuang and Ying Sheng and Lianmin Zheng and Cody Hao Yu and Joseph E. Gonzalez and Hao Zhang and Ion Stoica},
      year={2023},
      eprint={2309.06180},
      archivePrefix={arXiv},
      primaryClass={cs.LG},
      url={https://arxiv.org/abs/2309.06180}, 
}

@inproceedings{Sheng_2025, series={EuroSys ’25},
   title={HybridFlow: A Flexible and Efficient RLHF Framework},
   url={http://dx.doi.org/10.1145/3689031.3696075},
   DOI={10.1145/3689031.3696075},
   booktitle={Proceedings of the Twentieth European Conference on Computer Systems},
   publisher={ACM},
   author={Sheng, Guangming and Zhang, Chi and Ye, Zilingfeng and Wu, Xibin and Zhang, Wang and Zhang, Ru and Peng, Yanghua and Lin, Haibin and Wu, Chuan},
   year={2025},
   month=mar, pages={1279–1297},
   collection={EuroSys ’25} }

@misc{lambert2024tulu,
      title={Tulu 3: Pushing Frontiers in Open Language Model Post-Training}, 
      author={Nathan Lambert and Jacob Morrison and Valentina Pyatkin and Shengyi Huang and Hamish Ivison and Faeze Brahman and Lester James V. Miranda and Alisa Liu and Nouha Dziri and Shane Lyu and Yuling Gu and Saumya Malik and Victoria Graf and Jena D. Hwang and Jiangjiang Yang and Ronan Le Bras and Oyvind Tafjord and Chris Wilhelm and Luca Soldaini and Noah A. Smith and Yizhong Wang and Pradeep Dasigi and Hannaneh Hajishirzi},
      year={2025},
      eprint={2411.15124},
      archivePrefix={arXiv},
      primaryClass={cs.CL},
      url={https://arxiv.org/abs/2411.15124}, 
}

@book{bishop2006pattern,
  author       = {Christopher M. Bishop},
  title        = {Pattern recognition and machine learning, 5th Edition},
  series       = {Information science and statistics},
  publisher    = {Springer},
  year         = {2007},
  url          = {https://www.worldcat.org/oclc/71008143},
  isbn         = {9780387310732},
  bibsource    = {dblp computer science bibliography, https://dblp.org}
}

@misc{yu2025dapoopensourcellmreinforcement,
      title={DAPO: An Open-Source LLM Reinforcement Learning System at Scale}, 
      author={Qiying Yu and Zheng Zhang and Ruofei Zhu and Yufeng Yuan and Xiaochen Zuo and Yu Yue and Weinan Dai and Tiantian Fan and Gaohong Liu and Lingjun Liu and Xin Liu and Haibin Lin and Zhiqi Lin and Bole Ma and Guangming Sheng and Yuxuan Tong and Chi Zhang and Mofan Zhang and Wang Zhang and Hang Zhu and Jinhua Zhu and Jiaze Chen and Jiangjie Chen and Chengyi Wang and Hongli Yu and Yuxuan Song and Xiangpeng Wei and Hao Zhou and Jingjing Liu and Wei-Ying Ma and Ya-Qin Zhang and Lin Yan and Mu Qiao and Yonghui Wu and Mingxuan Wang},
      year={2025},
      eprint={2503.14476},
      archivePrefix={arXiv},
      primaryClass={cs.LG},
      url={https://arxiv.org/abs/2503.14476}, 
}

@misc{murphy2025reinforcementlearningoverview,
      title={Reinforcement Learning: An Overview}, 
      author={Kevin Murphy},
      year={2025},
      eprint={2412.05265},
      archivePrefix={arXiv},
      primaryClass={cs.AI},
      url={https://arxiv.org/abs/2412.05265}, 
}

@book{Sutton1998,
  author       = {Richard S. Sutton and
                  Andrew G. Barto},
  title        = {Reinforcement learning - an introduction, 2nd Edition},
  publisher    = {{MIT} Press},
  year         = {2018},
  url          = {http://www.incompleteideas.net/book/the-book-2nd.html},
  bibsource    = {dblp computer science bibliography, https://dblp.org}
}

@misc{schulman2017proximalpolicyoptimizationalgorithms,
      title={Proximal Policy Optimization Algorithms}, 
      author={John Schulman and Filip Wolski and Prafulla Dhariwal and Alec Radford and Oleg Klimov},
      year={2017},
      eprint={1707.06347},
      archivePrefix={arXiv},
      primaryClass={cs.LG},
      url={https://arxiv.org/abs/1707.06347}, 
}

@misc{andrychowicz2018hindsightexperiencereplay,
      title={Hindsight Experience Replay}, 
      author={Marcin Andrychowicz and Filip Wolski and Alex Ray and Jonas Schneider and Rachel Fong and Peter Welinder and Bob McGrew and Josh Tobin and Pieter Abbeel and Wojciech Zaremba},
      year={2018},
      eprint={1707.01495},
      archivePrefix={arXiv},
      primaryClass={cs.LG},
      url={https://arxiv.org/abs/1707.01495}, 
}

@misc{su2025trustregionadaptivepolicyoptimization,
      title={Trust-Region Adaptive Policy Optimization}, 
      author={Mingyu Su and Jian Guan and Yuxian Gu and Minlie Huang and Hongning Wang},
      year={2025},
      eprint={2512.17636},
      archivePrefix={arXiv},
      primaryClass={cs.LG},
      url={https://arxiv.org/abs/2512.17636}, 
}

@misc{ma2025learningreinforcementlearningcant,
      title={Learning What Reinforcement Learning Can't: Interleaved Online Fine-Tuning for Hardest Questions}, 
      author={Lu Ma and Hao Liang and Meiyi Qiang and Lexiang Tang and Xiaochen Ma and Zhen Hao Wong and Junbo Niu and Chengyu Shen and Runming He and Yanhao Li and Bin Cui and Wentao Zhang},
      year={2025},
      eprint={2506.07527},
      archivePrefix={arXiv},
      primaryClass={cs.AI},
      url={https://arxiv.org/abs/2506.07527}, 
}

@misc{liu2025uftunifyingsupervisedreinforcement,
      title={UFT: Unifying Supervised and Reinforcement Fine-Tuning}, 
      author={Mingyang Liu and Gabriele Farina and Asuman Ozdaglar},
      year={2025},
      eprint={2505.16984},
      archivePrefix={arXiv},
      primaryClass={cs.LG},
      url={https://arxiv.org/abs/2505.16984}, 
}

@misc{fu2025srftsinglestagemethodsupervised,
      title={SRFT: A Single-Stage Method with Supervised and Reinforcement Fine-Tuning for Reasoning}, 
      author={Yuqian Fu and Tinghong Chen and Jiajun Chai and Xihuai Wang and Songjun Tu and Guojun Yin and Wei Lin and Qichao Zhang and Yuanheng Zhu and Dongbin Zhao},
      year={2025},
      eprint={2506.19767},
      archivePrefix={arXiv},
      primaryClass={cs.CL},
      url={https://arxiv.org/abs/2506.19767}, 
}

@misc{yan2025learningreasonoffpolicyguidance,
      title={Learning to Reason under Off-Policy Guidance}, 
      author={Jianhao Yan and Yafu Li and Zican Hu and Zhi Wang and Ganqu Cui and Xiaoye Qu and Yu Cheng and Yue Zhang},
      year={2025},
      eprint={2504.14945},
      archivePrefix={arXiv},
      primaryClass={cs.LG},
      url={https://arxiv.org/abs/2504.14945}, 
}

@misc{lv2026unifiedviewlargelanguage,
      title={Towards a Unified View of Large Language Model Post-Training}, 
      author={Xingtai Lv and Yuxin Zuo and Youbang Sun and Hongyi Liu and Yuntian Wei and Zhekai Chen and Xuekai Zhu and Kaiyan Zhang and Bingning Wang and Ning Ding and Bowen Zhou},
      year={2026},
      eprint={2509.04419},
      archivePrefix={arXiv},
      primaryClass={cs.LG},
      url={https://arxiv.org/abs/2509.04419}, 
}

@misc{zhang2025onpolicyrlmeetsoffpolicy,
      title={On-Policy RL Meets Off-Policy Experts: Harmonizing Supervised Fine-Tuning and Reinforcement Learning via Dynamic Weighting}, 
      author={Wenhao Zhang and Yuexiang Xie and Yuchang Sun and Yanxi Chen and Guoyin Wang and Yaliang Li and Bolin Ding and Jingren Zhou},
      year={2025},
      eprint={2508.11408},
      archivePrefix={arXiv},
      primaryClass={cs.LG},
      url={https://arxiv.org/abs/2508.11408}, 
}

@misc{yoshihara2025practicaltwostagerecipemathematical,
      title={A Practical Two-Stage Recipe for Mathematical LLMs: Maximizing Accuracy with SFT and Efficiency with Reinforcement Learning}, 
      author={Hiroshi Yoshihara and Taiki Yamaguchi and Yuichi Inoue},
      year={2025},
      eprint={2507.08267},
      archivePrefix={arXiv},
      primaryClass={cs.LG},
      url={https://arxiv.org/abs/2507.08267}, 
}

@misc{wei2022finetunedlanguagemodelszeroshot,
      title={Finetuned Language Models Are Zero-Shot Learners}, 
      author={Jason Wei and Maarten Bosma and Vincent Y. Zhao and Kelvin Guu and Adams Wei Yu and Brian Lester and Nan Du and Andrew M. Dai and Quoc V. Le},
      year={2022},
      eprint={2109.01652},
      archivePrefix={arXiv},
      primaryClass={cs.CL},
      url={https://arxiv.org/abs/2109.01652}, 
}

@misc{ouyang2022traininglanguagemodelsfollow,
      title={Training language models to follow instructions with human feedback}, 
      author={Long Ouyang and Jeff Wu and Xu Jiang and Diogo Almeida and Carroll L. Wainwright and Pamela Mishkin and Chong Zhang and Sandhini Agarwal and Katarina Slama and Alex Ray and John Schulman and Jacob Hilton and Fraser Kelton and Luke Miller and Maddie Simens and Amanda Askell and Peter Welinder and Paul Christiano and Jan Leike and Ryan Lowe},
      year={2022},
      eprint={2203.02155},
      archivePrefix={arXiv},
      primaryClass={cs.CL},
      url={https://arxiv.org/abs/2203.02155}, 
}

@misc{zeng2025simplerlzooinvestigatingtamingzero,
      title={SimpleRL-Zoo: Investigating and Taming Zero Reinforcement Learning for Open Base Models in the Wild}, 
      author={Weihao Zeng and Yuzhen Huang and Qian Liu and Wei Liu and Keqing He and Zejun Ma and Junxian He},
      year={2025},
      eprint={2503.18892},
      archivePrefix={arXiv},
      primaryClass={cs.LG},
      url={https://arxiv.org/abs/2503.18892}, 
}

@misc{yue2025doesreinforcementlearningreally,
      title={Does Reinforcement Learning Really Incentivize Reasoning Capacity in LLMs Beyond the Base Model?}, 
      author={Yang Yue and Zhiqi Chen and Rui Lu and Andrew Zhao and Zhaokai Wang and Yang Yue and Shiji Song and Gao Huang},
      year={2025},
      eprint={2504.13837},
      archivePrefix={arXiv},
      primaryClass={cs.AI},
      url={https://arxiv.org/abs/2504.13837}, 
}

@misc{tie2025surveyposttraininglargelanguage,
      title={A Survey on Post-training of Large Language Models}, 
      author={Guiyao Tie and Zeli Zhao and Dingjie Song and Fuyang Wei and Rong Zhou and Yurou Dai and Wen Yin and Zhejian Yang and Jiangyue Yan and Yao Su and Zhenhan Dai and Yifeng Xie and Yihan Cao and Lichao Sun and Pan Zhou and Lifang He and Hechang Chen and Yu Zhang and Qingsong Wen and Tianming Liu and Neil Zhenqiang Gong and Jiliang Tang and Caiming Xiong and Heng Ji and Philip S. Yu and Jianfeng Gao},
      year={2025},
      eprint={2503.06072},
      archivePrefix={arXiv},
      primaryClass={cs.CL},
      url={https://arxiv.org/abs/2503.06072}, 
}

@misc{zeng2025simplerl,
      title={SimpleRL-Zoo: Investigating and Taming Zero Reinforcement Learning for Open Base Models in the Wild}, 
      author={Weihao Zeng and Yuzhen Huang and Qian Liu and Wei Liu and Keqing He and Zejun Ma and Junxian He},
      year={2025},
      eprint={2503.18892},
      archivePrefix={arXiv},
      primaryClass={cs.LG},
      url={https://arxiv.org/abs/2503.18892}, 
}

@misc{DBLP:conf/acl/HeLBHTSHHHZLQL024,
      title={OlympiadBench: A Challenging Benchmark for Promoting AGI with Olympiad-Level Bilingual Multimodal Scientific Problems}, 
      author={Chaoqun He and Renjie Luo and Yuzhuo Bai and Shengding Hu and Zhen Leng Thai and Junhao Shen and Jinyi Hu and Xu Han and Yujie Huang and Yuxiang Zhang and Jie Liu and Lei Qi and Zhiyuan Liu and Maosong Sun},
      year={2024},
      eprint={2402.14008},
      archivePrefix={arXiv},
      primaryClass={cs.CL},
      url={https://arxiv.org/abs/2402.14008}, 
}

@misc{DBLP:conf/nips/LewkowyczADDMRS22,
      title={Solving Quantitative Reasoning Problems with Language Models}, 
      author={Aitor Lewkowycz and Anders Andreassen and David Dohan and Ethan Dyer and Henryk Michalewski and Vinay Ramasesh and Ambrose Slone and Cem Anil and Imanol Schlag and Theo Gutman-Solo and Yuhuai Wu and Behnam Neyshabur and Guy Gur-Ari and Vedant Misra},
      year={2022},
      eprint={2206.14858},
      archivePrefix={arXiv},
      primaryClass={cs.CL},
      url={https://arxiv.org/abs/2206.14858}, 
}

@misc{openai2025gptoss120bgptoss20bmodel,
      title={gpt-oss-120b \& gpt-oss-20b Model Card}, 
      author={OpenAI and : and Sandhini Agarwal and Lama Ahmad and Jason Ai and Sam Altman and Andy Applebaum and Edwin Arbus and Rahul K. Arora and Yu Bai and Bowen Baker and Haiming Bao and Boaz Barak and Ally Bennett and Tyler Bertao and Nivedita Brett and Eugene Brevdo and Greg Brockman and Sebastien Bubeck and Che Chang and Kai Chen and Mark Chen and Enoch Cheung and Aidan Clark and Dan Cook and Marat Dukhan and Casey Dvorak and Kevin Fives and Vlad Fomenko and Timur Garipov and Kristian Georgiev and Mia Glaese and Tarun Gogineni and Adam Goucher and Lukas Gross and Katia Gil Guzman and John Hallman and Jackie Hehir and Johannes Heidecke and Alec Helyar and Haitang Hu and Romain Huet and Jacob Huh and Saachi Jain and Zach Johnson and Chris Koch and Irina Kofman and Dominik Kundel and Jason Kwon and Volodymyr Kyrylov and Elaine Ya Le and Guillaume Leclerc and James Park Lennon and Scott Lessans and Mario Lezcano-Casado and Yuanzhi Li and Zhuohan Li and Ji Lin and Jordan Liss and Lily and Liu and Jiancheng Liu and Kevin Lu and Chris Lu and Zoran Martinovic and Lindsay McCallum and Josh McGrath and Scott McKinney and Aidan McLaughlin and Song Mei and Steve Mostovoy and Tong Mu and Gideon Myles and Alexander Neitz and Alex Nichol and Jakub Pachocki and Alex Paino and Dana Palmie and Ashley Pantuliano and Giambattista Parascandolo and Jongsoo Park and Leher Pathak and Carolina Paz and Ludovic Peran and Dmitry Pimenov and Michelle Pokrass and Elizabeth Proehl and Huida Qiu and Gaby Raila and Filippo Raso and Hongyu Ren and Kimmy Richardson and David Robinson and Bob Rotsted and Hadi Salman and Suvansh Sanjeev and Max Schwarzer and D. Sculley and Harshit Sikchi and Kendal Simon and Karan Singhal and Yang Song and Dane Stuckey and Zhiqing Sun and Philippe Tillet and Sam Toizer and Foivos Tsimpourlas and Nikhil Vyas and Eric Wallace and Xin Wang and Miles Wang and Olivia Watkins and Kevin Weil and Amy Wendling and Kevin Whinnery and Cedric Whitney and Hannah Wong and Lin Yang and Yu Yang and Michihiro Yasunaga and Kristen Ying and Wojciech Zaremba and Wenting Zhan and Cyril Zhang and Brian Zhang and Eddie Zhang and Shengjia Zhao},
      year={2025},
      eprint={2508.10925},
      archivePrefix={arXiv},
      primaryClass={cs.CL},
      url={https://arxiv.org/abs/2508.10925}, 
}

@misc{DBLP:conf/nips/HendrycksBKABTS21,
      title={Measuring Mathematical Problem Solving With the MATH Dataset}, 
      author={Dan Hendrycks and Collin Burns and Saurav Kadavath and Akul Arora and Steven Basart and Eric Tang and Dawn Song and Jacob Steinhardt},
      year={2021},
      eprint={2103.03874},
      archivePrefix={arXiv},
      primaryClass={cs.LG},
      url={https://arxiv.org/abs/2103.03874}, 
}

@misc{li2024numinamath,
  author = {Jia LI and Edward Beeching and Lewis Tunstall and Ben Lipkin and Roman Soletskyi and Shengyi Costa Huang and Kashif Rasul and Longhui Yu and Albert Jiang and Ziju Shen and Zihan Qin and Bin Dong and Li Zhou and Yann Fleureau and Guillaume Lample and Stanislas Polu},
  title = {NuminaMath},
  year = {2024},
  publisher = {Numina},
  journal = {Hugging Face repository},
  howpublished = {\url{[https://huggingface.co/AI-MO/NuminaMath-CoT](https://github.com/project-numina/aimo-progress-prize/blob/main/report/numina_dataset.pdf)}}
}

@misc{liu2025understanding,
      title={Understanding R1-Zero-Like Training: A Critical Perspective}, 
      author={Zichen Liu and Changyu Chen and Wenjun Li and Penghui Qi and Tianyu Pang and Chao Du and Wee Sun Lee and Min Lin},
      year={2025},
      eprint={2503.20783},
      archivePrefix={arXiv},
      primaryClass={cs.LG},
      url={https://arxiv.org/abs/2503.20783}, 
}

@misc{shao2024deepseekmath,
      title={DeepSeekMath: Pushing the Limits of Mathematical Reasoning in Open Language Models}, 
      author={Zhihong Shao and Peiyi Wang and Qihao Zhu and Runxin Xu and Junxiao Song and Xiao Bi and Haowei Zhang and Mingchuan Zhang and Y. K. Li and Y. Wu and Daya Guo},
      year={2024},
      eprint={2402.03300},
      archivePrefix={arXiv},
      primaryClass={cs.CL},
      url={https://arxiv.org/abs/2402.03300}, 
}

@misc{huang2025blending,
      title={Blending Supervised and Reinforcement Fine-Tuning with Prefix Sampling}, 
      author={Zeyu Huang and Tianhao Cheng and Zihan Qiu and Zili Wang and Yinghui Xu and Edoardo M. Ponti and Ivan Titov},
      year={2025},
      eprint={2507.01679},
      archivePrefix={arXiv},
      primaryClass={cs.LG},
      url={https://arxiv.org/abs/2507.01679}, 
}

@misc{zhao2020energybasedhindsightexperienceprioritization,
      title={Energy-Based Hindsight Experience Prioritization}, 
      author={Rui Zhao and Volker Tresp},
      year={2020},
      eprint={1810.01363},
      archivePrefix={arXiv},
      primaryClass={cs.LG},
      url={https://arxiv.org/abs/1810.01363}, 
}

@misc{rauber2019hindsightpolicygradients,
      title={Hindsight policy gradients}, 
      author={Paulo Rauber and Avinash Ummadisingu and Filipe Mutz and Juergen Schmidhuber},
      year={2019},
      eprint={1711.06006},
      archivePrefix={arXiv},
      primaryClass={cs.LG},
      url={https://arxiv.org/abs/1711.06006}, 
}

@misc{liu2024oat,
  title={OAT: A research-friendly framework for LLM online alignment},
  author={Liu, Zichen and Chen, Changyu and Wan, Xinyi and Du, Chao and Lee, Wee Sun and Lin, Min},
  year={2024},
  url={https://github.com/sail-sg/oat},
}

@misc{shenfeld2025rlsrazoronlinereinforcement,
      title={RL's Razor: Why Online Reinforcement Learning Forgets Less}, 
      author={Idan Shenfeld and Jyothish Pari and Pulkit Agrawal},
      year={2025},
      eprint={2509.04259},
      archivePrefix={arXiv},
      primaryClass={cs.LG},
      url={https://arxiv.org/abs/2509.04259}, 
}

@article{qwen2025qwen25technicalreport,
  title={Qwen2.5 Technical Report},
  author={Qwen An Yang and Baosong Yang and Beichen Zhang and Binyuan Hui and Bo Zheng and Bowen Yu and Chengyuan Li and Dayiheng Liu and Fei Huang and Guanting Dong and Haoran Wei and Huan Lin and Jian Yang and Jianhong Tu and Jianwei Zhang and Jianxin Yang and Jiaxin Yang and Jingren Zhou and Junyang Lin and Kai Dang and Keming Lu and Keqin Bao and Kexin Yang and Le Yu and Mei Li and Mingfeng Xue and Pei Zhang and Qin Zhu and Rui Men and Runji Lin and Tianhao Li and Tingyu Xia and Xingzhang Ren and Xuancheng Ren and Yang Fan and Yang Su and Yi-Chao Zhang and Yunyang Wan and Yuqi Liu and Zeyu Cui and Zhenru Zhang and Zihan Qiu and Shanghaoran Quan and Zekun Wang},
  journal={ArXiv},
  year={2024},
  volume={abs/2412.15115},
  url={https://api.semanticscholar.org/CorpusID:274859421}
}

@misc{zhang2025surveyreinforcementlearninglarge,
      title={A Survey of Reinforcement Learning for Large Reasoning Models}, 
      author={Kaiyan Zhang and Yuxin Zuo and Bingxiang He and Youbang Sun and Runze Liu and Che Jiang and Yuchen Fan and Kai Tian and Guoli Jia and Pengfei Li and Yu Fu and Xingtai Lv and Yuchen Zhang and Sihang Zeng and Shang Qu and Haozhan Li and Shijie Wang and Yuru Wang and Xinwei Long and Fangfu Liu and Xiang Xu and Jiaze Ma and Xuekai Zhu and Ermo Hua and Yihao Liu and Zonglin Li and Huayu Chen and Xiaoye Qu and Yafu Li and Weize Chen and Zhenzhao Yuan and Junqi Gao and Dong Li and Zhiyuan Ma and Ganqu Cui and Zhiyuan Liu and Biqing Qi and Ning Ding and Bowen Zhou},
      year={2025},
      eprint={2509.08827},
      archivePrefix={arXiv},
      primaryClass={cs.CL},
      url={https://arxiv.org/abs/2509.08827}, 
}
